\documentclass[11pt]{article}
\usepackage[margin=1in]{geometry}
\usepackage[T1]{fontenc}
\usepackage{lmodern}
\usepackage{amsmath,amssymb,amsthm}
\usepackage{booktabs}
\usepackage{graphicx}
\usepackage[numbers]{natbib}
\usepackage{microtype}
\usepackage[hidelinks]{hyperref}
\usepackage{xurl}
\newtheorem{theorem}{Theorem}

\newtheorem{corollary}{Corollary}
\title{Attention-Weighted Value Projection for KV-Cache Compression}
\author{Samuel Salfati\\fraQtl AI Research\\\texttt{samuel@fraqtl.ai}}
\date{September 2026}
\begin{document}
\maketitle

\begin{abstract}
Rank reduction discards dimensions; quantization keeps them at lower
precision. Comparing the two requires a choice of what compression
should preserve. For attention values, we study reconstruction of the
attention output rather than reconstruction of the values alone. With
fixed attention weights, the optimal orthogonal rank-$r$ projection
uses the leading eigenvectors of $V^\top\alpha^\top\alpha V$, and its
error is exactly the discarded eigenvalue sum. We extend this objective
to calibration datasets and grouped query attention, and describe rank
allocation under an additive local error budget. We also examine the
limits of using local error to predict downstream loss. Historical
weight-perturbation experiments favor coefficient rounding over the
value projections tested, but do not establish a comparison at equal
cache storage. The resulting distinction is practical: the projection
objective has an exact solution, while a comparison with cache
quantization requires separate activation-level experiments.
\end{abstract}

\section{Introduction}
The KV cache stores key and value activations for all previous tokens
during transformer inference. For $L$ layers, $H$ KV heads, head
dimension $d$, and sequence length $T$, it contains $2LHdT$ elements
when keys and values have equal dimensions. Two strategies for reducing
this cost are \textit{rank reduction}, which projects activations to
fewer dimensions, and \textit{quantization}, which represents them at
lower precision.

The choice of subspace matters. Uncentered PCA preserves directions with high
energy in $V$, but the next operation consumes $\alpha V$. A direction
that contributes little to the values may matter differently after
attention mixes tokens. The natural local question is therefore:
which rank-$r$ value projection best preserves the attention output?

For fixed attention weights, this question has an exact answer. The
right basis comes from $V^\top\alpha^\top\alpha V$. This is uncentered PCA on the
attention output $\alpha V$, not on $V$ itself. Its eigenvalue tail
measures the local error at each rank. The same construction applies
when several query heads share a value head, provided their individual
output errors are added before choosing the basis.

We develop this result, its calibration and rank-allocation forms,
and its limits. We then return to the comparison with quantization.
The available experiments perturb projection weights and reconstruct
dense floating-point weights. Their perplexities describe those
interventions; they cannot establish which cache codec gives better
quality at equal storage. Keeping this distinction explicit leaves a
well-defined projection result and an empirical question that still
requires a cache-level comparison.

\paragraph{Version note.}
This version replaces the general downstream-optimality and
matched-cache-storage claims of arXiv:2604.11501v1 with the local
result and qualified observations below. Appendix~\ref{app:corrections}
details the changes. No new model evaluations are reported.

\section{Background and Related Work}
KIVI uses asymmetric cache quantization, including different grouping
for keys and values~\citep{kivi}; a simple weight-rounding baseline
is not an implementation of KIVI. Palu includes low-rank projection,
rank search, quantization compatibility, and optimized execution
\citep{palu}. CSKV combines compressed features with a window-based
cache~\citep{cskv}, and ReCalKV studies head reordering and offline
calibration~\citep{recalkv}. These methods are outside the scope of a
comparison against one fixed per-head projection implementation.

KQ-SVD also considers value/output interactions. Its value objective
optimizes a factorization $VAB^\top W^O$, not merely an orthogonal
projector $VPP^\top W^O$~\citep{kqsvd}. Substituting the latter
constraint does not evaluate the full method. The V theorem here
uses a different local objective, $\alpha V$, and does not establish
superiority to KQ-SVD. These objectives and feasible sets should be kept distinct when
comparing compression methods.

\section{Attention-Weighted Value Projection}
\label{sec:theorem}
For calibration sequence $s$, let $V_s\in\mathbb R^{T_s\times d}$.
Let $\alpha_{s,h}\in\mathbb R^{T_s\times T_s}$ be the fixed attention
matrix for query head $h$ sharing this value head. Choose fixed
weights $w_{s,h}\geq0$ and an orthogonal projector $\Pi=PP^\top$, with
$P\in\mathbb R^{d\times r}$ and $P^\top P=I_r$. Define
\begin{align}
 D(P)&=\sum_{s,h}w_{s,h}\|\alpha_{s,h}V_s(I-PP^\top)\|_F^2,\label{eq:objective}\\
 M_V&=\sum_{s,h}w_{s,h}V_s^\top\alpha_{s,h}^\top\alpha_{s,h}V_s.
 \label{eq:gram}
\end{align}
Weights can specify a sum or a normalized average; they must be
reported when comparing spectra across heads, layers, or datasets.

\begin{theorem}[Attention-weighted value projection]
Let $\lambda_1\geq\cdots\geq\lambda_d\geq0$ be the eigenvalues of
$M_V$. A matrix $P$ spanning a top-$r$ eigenspace minimizes
\eqref{eq:objective}, and the minimum distortion is
\[
 \min_{P^\top P=I_r}D(P)=\sum_{i=r+1}^d\lambda_i.
\]
The optimum need not be unique when eigenvalues tie.
\end{theorem}
\begin{proof}
Put $R=I-PP^\top$. Since $R^\top=R$ and $R^2=R$,
\[
 D(P)=\operatorname{Tr}(RM_VR)=\operatorname{Tr}(M_VR)
 =\operatorname{Tr}(M_V)-\operatorname{Tr}(P^\top M_VP).
\]
The Rayleigh--Ritz variational principle gives the maximum of the
last trace as $\sum_{i=1}^r\lambda_i$, proving the result.
\end{proof}

For one sequence and one head, $M_V=V^\top\alpha^\top\alpha V$.
Equivalently, the solution uses the right singular vectors of
$\alpha V$. The eigenspace optimization is classical. Here it identifies the
subspace that preserves the chosen attention-output objective.

\subsection{A case where value energy selects the wrong direction}
The two objectives can disagree even in two dimensions. Consider
\[
 V=\begin{pmatrix}3&1\\-3&1\end{pmatrix},\qquad
 \alpha=\frac12\begin{pmatrix}1&1\\1&1\end{pmatrix}.
\]
This is a constructed example with unmasked uniform attention, not a
measurement from a causal language model. Here
\[
 V^\top V=\operatorname{diag}(18,2),\qquad
 M_V=\operatorname{diag}(0,2).
\]
Uncentered value PCA retains the first coordinate at rank one, losing
the entire attention output and incurring squared error $2$. The
attention-weighted basis retains the second coordinate and reconstructs
the output exactly. Attention averages away the first direction, even
though it contains most of the value energy. This example establishes
that the objectives can differ; it does not estimate how often they
differ in a model.

\begin{figure}[t]
\centering
\includegraphics[width=\linewidth]{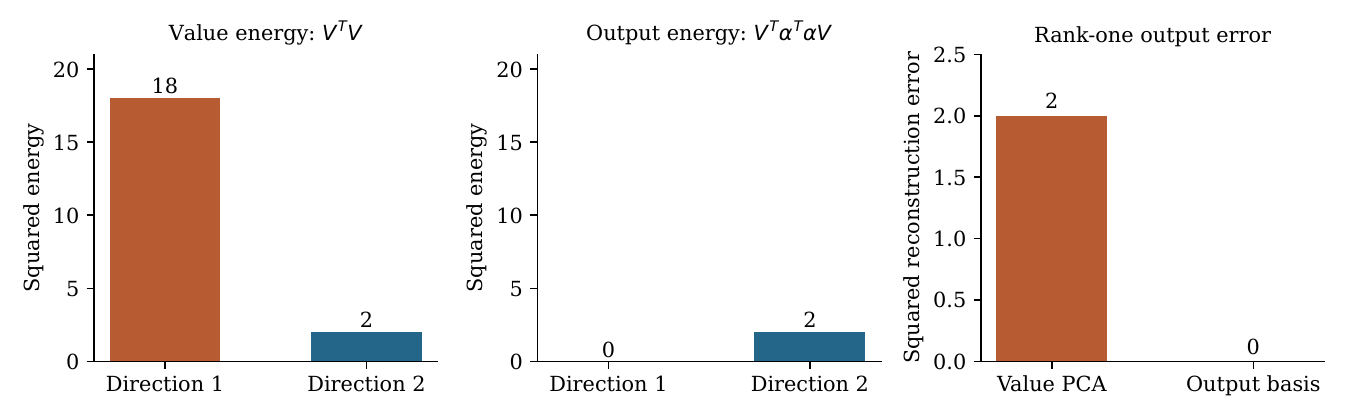}
\caption{An exact two-dimensional example. The direction with the most
value energy contributes nothing after uniform token averaging.
Choosing the basis from the attention output gives zero rank-one
output error. This constructed example uses unmasked attention.}
\label{fig:objective}
\end{figure}

\paragraph{GQA.}
Equation~\eqref{eq:gram} sums individual query-head Gram matrices.
Replacing that sum by the Gram matrix of the mean attention introduces
cross-head terms and generally changes the optimizer. 

\paragraph{Calibration and deployment.}
The theorem is exact for the given calibration inputs and attention
matrices. A population statement follows by taking expectations when
the required second moments exist. Empirical optimality does not
guarantee generalization to new tokens, contexts, or domains. A basis
fitted from a completed evaluation sequence is an oracle for that
sequence, not a static calibration result.

\paragraph{Output projection and layer interactions.}
The objective precedes the output projection $W^O$. In general,
minimizing $\|O(I-PP^\top)W^O\|_F^2$ is a different problem. For
example, let $O=\operatorname{diag}(10,1)$ and
$W^O=\operatorname{diag}(10^{-3},1)$. Retaining the first coordinate
minimizes local output error, but incurs squared error $1$ after
$W^O$; retaining the second incurs only $10^{-4}$ there. Thus the
local optimum need not minimize downstream error even for a linear
output map. Compressing values in earlier layers can also change the
queries and keys of later layers. Simultaneous multilayer projection
is not globally solved by independent applications of the theorem.

\paragraph{Biases and centering.}
For column-vector activations $v=W_vx+b_v$, projecting values requires
$\widehat v=\Pi W_vx+\Pi b_v$. Projecting only weights leaves an
unprojected bias. This distinction affects the GPT-2 weight interventions discussed
in Appendix~\ref{app:corrections}. An affine
reconstruction $\widehat V=\mathbf1\mu^\top+(V-\mathbf1\mu^\top)\Pi$
has the same theorem for the centered values with fixed $\mu$,
provided the mean is restored and accounted for. Uncentered second
moments must not be described as centered covariance without this
distinction.

\subsection{Rank allocation under an additive local objective}
For groups $g$ with fixed spectra and per-direction storage costs
$c_g$, one can optimize
\[
 \min_{\{r_g\}}\sum_g\sum_{i>r_g}\lambda_{g,i}
 \quad\text{subject to}\quad\sum_g c_g r_g\leq B.
\]
When all costs are equal, retaining the largest marginal eigenvalues
globally solves this additive objective; ties can be resolved to
preserve each group's spectral prefix. Unequal integer costs require
a constrained discrete allocation procedure, such as dynamic
programming; greedy sorting by $\lambda/c$ is not generally exact.
Basis and metadata costs, rank granularity, and cross-layer effects
must be included in any deployed allocation model. The allocation solves the stated sum of local errors. End-to-end
perplexity and mixed-precision allocation require different objectives.

\begin{corollary}[Equal-cost rank allocation]
Suppose each retained direction has the same cost, the budget permits
$B$ directions, and the objective is the sum of the fixed local
spectral tails. An optimal allocation retains the $B$ largest
eigenvalues across all groups, resolving ties to retain a prefix in
each group.
\end{corollary}
\begin{proof}
The total error equals the sum of all eigenvalues minus the sum of
retained eigenvalues. No choice of $B$ values can exceed the sum of the
$B$ largest. Since each group's eigenvalues are ordered, that choice
can satisfy the prefix constraint, including when values tie.
\end{proof}

\begin{figure}[t]
\centering
\includegraphics[width=.95\linewidth]{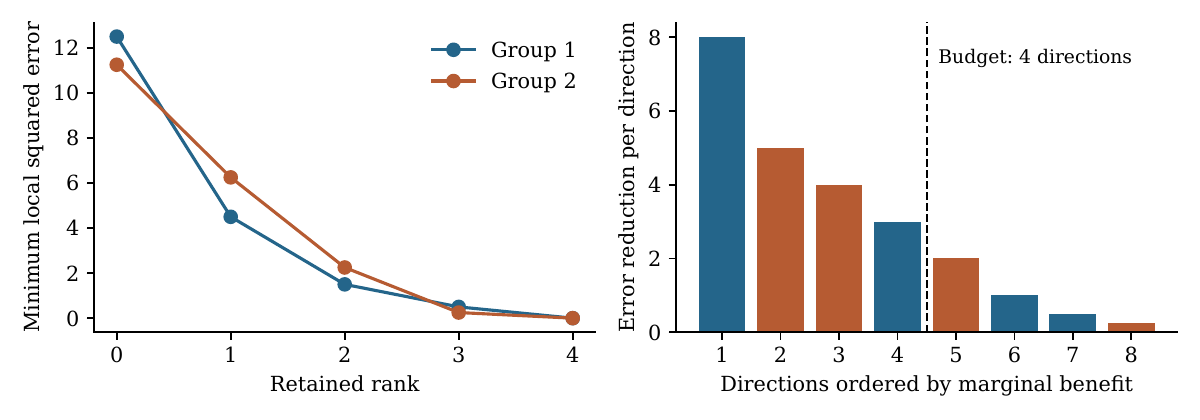}
\caption{The spectral tail gives the exact local rank--error curve
(left). With equal costs, sorting marginal eigenvalues allocates a
shared budget (right). The illustrative spectra are $(8,3,1,0.5)$ and
$(5,4,2,0.25)$; a budget of four directions retains two per group.
These are synthetic spectra, not measured model spectra.}
\label{fig:allocation}
\end{figure}

\subsection{From a basis to a cache representation}
Given a fitted basis $P$, the reduced representation is $Z=VP$.
The reconstructed attention output is
\[
 \widehat O=(\alpha Z)P^\top.
\]
Thus a value cache can store $r$ coefficients per token instead of $d$.
The basis is shared across tokens and must also be stored. If an output
projection follows, $P^\top W^O$ can be precomputed for the corresponding
head block. This identity specifies the representation; it does not
by itself provide a packed cache implementation or a faster kernel.

The Gram matrix can be accumulated from calibration outputs
$O_{s,h}=\alpha_{s,h}V_s$ as
$M_V=\sum_{s,h}w_{s,h}O_{s,h}^\top O_{s,h}$, followed by a symmetric
eigendecomposition. Calibration inputs, their weights, and the chosen
rank determine the fitted basis. Evaluation on separate sequences is
needed to determine how well that basis transfers.

\section{Local Error and Downstream Loss}
For $e=\operatorname{vec}(\Delta O)$, a task-loss expansion generally
contains $g^\top e+\tfrac12 e^\top H e$, as well as higher-order
terms. Training the model does not automatically make $g$ vanish
with respect to every internal activation. An alternative is the
KL divergence from the clean model's output distribution, whose
expansion at the clean output has no linear term. If its pullback
metric is approximated by a fixed $cI$, the quadratic surrogate is
proportional to the local reconstruction objective. This approximation needs to be justified for the model and data;
similar perplexities under two heuristics do not establish it.

\subsection{Local logit perturbations}
For $p=\operatorname{softmax}(s)$ and small $\delta$,
\[
 \mathrm{KL}(p\,\|\,\operatorname{softmax}(s+\delta))
 =\tfrac12\delta^\top G\delta+O(\|\delta\|^3),
 \qquad G=\operatorname{diag}(p)-pp^\top.
\]
Along a fixed direction $u$, a deterministic perturbation $a u$
has quadratic cost $a^2u^\top Gu/2$. Idealized zero-mean uniform
rounding noise of width $\Delta$ has expected quadratic cost
$\Delta^2u^\top Gu/24$. When $u^\top Gu>0$, their ratio is
$12a^2/\Delta^2$. Choosing $\Delta=2a/2^b$ gives $3\,2^{2b}$
under that particular convention. It is not a universal bound or
a matched-total-storage comparison. Actual ranges, clipping,
signal-dependent rounding, and finite deletion errors require
separate analysis. The ratio is undefined when both quadratic
terms vanish. Softmax is continuous: a change in its argmax can
occur with an arbitrarily small change in probabilities.

\section{Experiments}
\label{sec:audit}
\paragraph{Setup and interpretation.}
For column-vector conventions, the inspected weight-rounding
intervention is of the form
\[
 \widetilde W_v=U\mathcal Q_b(U^\top W_v),\qquad
 \widetilde v=\widetilde W_vx+b_v.
\]
Here $U$ is an orthogonal basis, and $\mathcal Q_b$ rounds coefficients
and immediately reconstructs floating-point values. A cache codec
would instead encode value activations, for example
$\widehat v=U\mathcal Q_b(U^\top v)$ with a specified scale and
packing scheme. In general $\mathcal Q_b(W)x\ne\mathcal Q_b(Wx)$.
The inspected main scripts restore dense floating-point weights and
do not implement a packed $b$-bit value cache.

Linear projection can be absorbed into weights, but dense projected
weights still do not implement reduced cache storage. A rank-$r$
codec must actually store the $r$-dimensional coefficients and
provide the necessary reconstruction or fused attention operations.
The earlier nominal $db$ versus $16r$ budget columns therefore
cannot establish equal physical cache memory for the experiments.

\begin{table}[t]
\centering\small
\caption{Selected historical WikiText-2 PPL measurements. ``Rounded''
means projection-weight rounding and floating-point reconstruction;
``rank'' means a value-weight projection intervention. These are
reported records, not new reproductions or matched-cache-byte results.}
\label{tab:historical}
\begin{tabular}{@{}lrrrr@{}}
\toprule
Model & Rank & Rank PPL & Rounding bits & Rounded PPL\\
\midrule
GPT-2 & 16 & 96.71 & 4 & 44.57\\
GPT-2 & 32 & 56.30 & 8 & 44.00\\
Llama-3.2 3B & 32 & 69.10 & 4 & 13.61\\
Llama-3.2 3B & 64 & 21.50 & 8 & 13.22\\
Mistral 7B & 32 & 44.19 & 4 & 9.42\\
Mistral 7B & 64 & 13.07 & 8 & 9.19\\
\bottomrule
\end{tabular}
\end{table}

\begin{figure}[t]
\centering
\includegraphics[width=\linewidth]{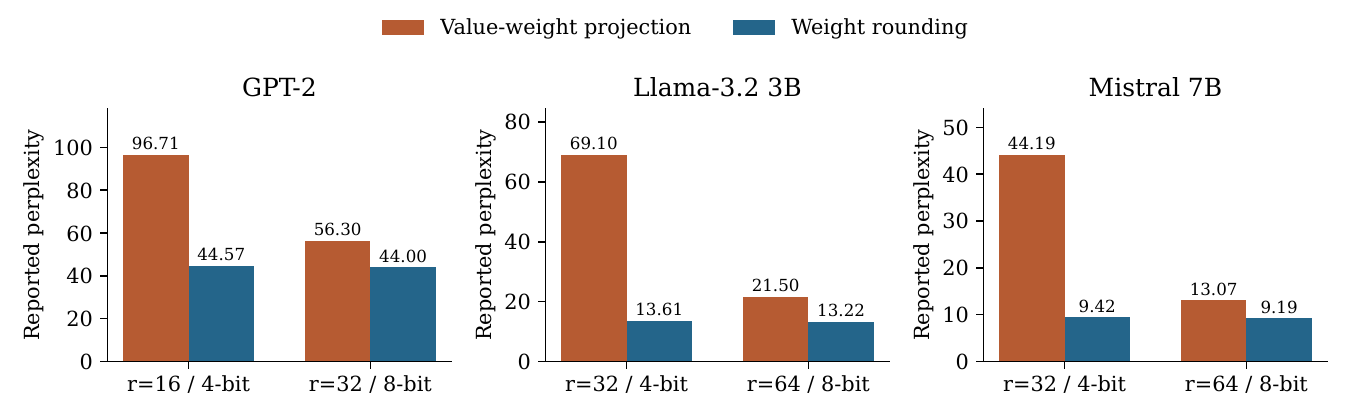}
\caption{Historical weight-intervention results from
Table~\ref{tab:historical}. Labels pair a projection rank with a
rounding precision; they do not denote equal measured cache storage.
Both methods reconstruct dense floating-point weights in the inspected
implementations. Values are reported records, not new reproductions.}
\label{fig:historical}
\end{figure}

\paragraph{Results.}
The reported perplexities are lower for coefficient rounding than for
the value-weight projections in every retained row of
Table~\ref{tab:historical}. The difference is large at the smaller
ranks: Mistral rank-32 gives 44.19 PPL, compared with 9.42 for four-bit
weight rounding. This motivates asking whether retaining more
dimensions at lower precision also helps when compressing the cache
itself. The table does not answer that question because the
interventions do not implement the corresponding cache representations.

The retained records cover GPT-2, Llama-3.2 3B, and Mistral 7B. Both
arms modify V weights in the inspected main scripts. They are
historical measurements, not new reproductions: complete per-example
outputs and executed-source manifests have not been recovered.
Other tables and downstream accuracy claims from v1 are excluded
pending consistent evaluation and provenance checks. The specific
protocol issues are recorded in Appendix~\ref{app:corrections}.

\section{Discussion}
A direct comparison requires applying both methods to actual K/V activations
under one causal protocol, using a common uncompressed model and
tokenizer. The experiment should record model revisions, data splits and sample
identities, calibration choices, per-example losses, and the
quantizer's grouping and scale rules. Basis fitting, rank or bit
selection, and final evaluation should use separate data. A
full-rank projection control must reproduce the baseline, including
bias handling and positional encoding.

Storage must include packed codes, scales, zero points where used,
projection bases, retained full-precision tokens, and other cache
metadata, with shared costs and temporary workspace reported
separately. A strong low-rank baseline and a documented cache
quantizer should be compared at equal accounted storage or on a
complete quality--storage frontier. Paired quality estimates should
accompany aggregate results. Long-context retrieval and serving
measurements are necessary before extending conclusions to those
settings. The present results do not measure serving performance.

\section{Conclusion}
For fixed attention weights, value projection is a spectral problem:
retain the leading directions of $V^\top\alpha^\top\alpha V$ and the
remaining eigenvalues give the exact local error. This provides a
basis for calibration and rank selection, with a clear boundary at
the local attention output. It does not determine downstream loss or
settle the choice between rank reduction and quantization. That
comparison remains an experimental question about actual cache
representations, their storage costs, and their effect on predictions.

\appendix
\section{Changes from the Original Version}
\label{app:corrections}
The original version extended the attention-specific result to general
downstream maps and interpreted weight-rounding experiments as
matched-storage cache comparisons. Both claims have been removed.
The V theorem is retained with a fixed-attention, orthogonal-projection,
local-output objective. The title and abstract now reflect that scope.
The following details explain the theoretical and experimental changes.

\subsection{Why the general theorem fails}
For differentiable $f$ and $R=I-PP^\top$, linearizing the output gives
the squared-error surrogate
\[
 D_{\rm lin}(P)=\mathbb E[x^\top R A(x)R x],
 \qquad A(x)=W^\top J_f(Wx)^\top J_f(Wx)W.
\]
Both copies of $R$ must remain. In general this is not
$\operatorname{Tr}(R\,\mathbb E[A(x)xx^\top])$; the proposed matrix
need not even be symmetric. With
\[
 R=\begin{pmatrix}1&0\\0&0\end{pmatrix},\quad
 A=\begin{pmatrix}1&1\\1&2\end{pmatrix},\quad x=\begin{pmatrix}1\\1\end{pmatrix},
\]
the two expressions are $1$ and $2$, respectively. The error occurs
even with constant positive-definite $A$, hence even for linear
downstream maps. A sufficient special case is constant $A$ and
$\mathbb E[xx^\top]=cI$, which yields $c\operatorname{Tr}(AR)$.
Linearity alone is insufficient. The claimed general unification of
other compression methods is consequently removed.

\subsection{Projection implementation and precision}
The inspected main GPT-2, Llama, and RoPE-corrected Mistral
comparisons modify V weights only on both sides. The earlier setup
describing K-PCA alongside every main rank-reduction comparison was
incorrect. GPT-2 is loaded at the library's default floating-point
precision in its script; the Mistral script explicitly requests
FP16 despite printing an ``FP32'' baseline label. A blanket precision
claim across these experiments is not justified.

The RoPE-corrected Mistral calibration script sums individual query-head
Grams, as required by Eq.~\eqref{eq:gram}. The earlier instruction to
average attention matrices before forming the Gram was incorrect.

\subsection{Joint K/V result and total-cache accounting}
The PE-01 script rounds \texttt{k\_proj} and \texttt{v\_proj} weights.
Its reported baseline of 9.19 and joint four-bit weight-rounding
PPL of 9.37 differ by $+0.18$. These numbers do not measure a joint
four-bit activation cache or its memory reduction. The same record
reports V-only weight rounding at 9.33; this differs from the
eigenbasis experiment at 9.42. Differences between these configurations
must not be attributed solely to adding key compression.

For a hypothetical packed cache with equal key/value dimension $d$,
FP16 K and V occupy $32d$ payload bits per token and KV head. FP16 K
plus INT4 V occupies $20d$ bits, a 37.5\% total reduction; INT4 K
plus INT4 V occupies $8d$ bits, a 75\% total reduction. These are
payload calculations, excluding scales, bases, residual windows,
alignment, and workspace. They are not measured savings from PE-01.

\subsection{Activation-rounding checks and causality}
Separate validation scripts round V activations in forward hooks.
Their scales use the maximum over the complete sequence (and batch,
when applicable), after which they return reconstructed floating-point
activations. These are distinct from the weight experiments and do
not implement packed cache storage. In teacher-forced next-token
evaluation, a complete-sequence scale can depend on future tokens,
so the resulting predictions need not obey the intended causal
protocol. A causal rerun must specify calibration-fixed scales,
past-only updates with a consistent re-encoding policy, or another
explicit streaming codec. This issue differs from compressing an
already observed prompt before generating its continuation.

\subsection{LAMBADA and split reconciliation}
The newer internal LAMBADA record reports GPT-2 baseline accuracy
7.4\% and rounded-weight accuracy 7.2\%; the older appendix baseline
24.2\% comes from a different evaluator. Mistral's corresponding
record reports 22.8\% and 23.0\%. Mixing those protocols was incorrect,
but merely substituting baselines would not repair the evaluation.
The newer evaluator constructs the context by collapsing whitespace,
then uses that context's token count to locate the target in the
independently tokenized original passage. It does not verify that
the context tokens are a prefix of the full tokens. This can misalign
the target. Skipped examples and normalization also differ between
evaluators. All LAMBADA parity and collapse claims are removed
pending a consistent evaluator and per-example records; no claim of
statistical equivalence is made from aggregate percentages.

The historical Mistral seed-summary result, rank-32 PPL
$28.83\pm0.86$, uses a different evaluation split with uncompressed
baseline 9.81, rather than the main experiment's 9.19. It must not
be interchanged with 44.19. Dividing a performance gap by one arm's
standard deviation is not a statistical significance test. Moreover,
the inspected routing script computes dispersion across eight
sequences at one seed, so its error bars cannot be relabeled as
three-seed variation without a separate supporting run.

\section{Audit provenance and availability}
The correction was informed by inspection of local source files and
research records listed below. Paths identify the audit trail; they
are not public download links or proof that a particular source
revision generated a reported number. The historical raw results
are not yet accompanied here by a complete public reproduction
package. No new external code/data release is claimed.
\begin{itemize}
\item Main GPT-2 interventions and activation hooks:
\path{notebooks/theory01_math/exp26d_gpt2_validation.py}.
\item Llama interventions and head aggregation:
\path{notebooks/theory01_math/exp26c_llama3_validation.py}.
\item RoPE-corrected Mistral and routing dispersion:
\path{notebooks/theory01_math/exp_audit_fixes.py}.
\item Joint weight rounding:
\path{notebooks/paper_experiments/pe01_k_quantization.py}.
\item Newer downstream evaluators:
\path{notebooks/theory01_math/exp_final_lambada_timing.py} and
\path{notebooks/theory01_math/exp_final_mistral_lambada.py}.
\item Older downstream evaluator:
\path{notebooks/kv95_lambada_fixed.py}.
\item Historical numerical summary and split descriptions:
\path{docs/ALL-RESULTS-FOR-PAPER.md}.
\end{itemize}

\section{Other interpretations removed from the original version}
Exploratory failures of selected algorithms do not demonstrate that
all rank-reduction methods fail. A small basis-ablation spread does
not prove basis independence across quantizers and models. A
Spearman correlation of 0.38 does not mean that 62\% of a ranking is
unexplained; a Pearson correlation of 0.69 does not imply 31\%
unexplained variance. Equal diagonal sensitivities in one basis do
not establish an isotropic full metric. Subspace overlap is not
statistical correlation. Keeping a high fraction of representation
energy does not by itself guarantee preservation of predictions,
but neither does a failed projection prove that variance is
unrelated to importance. These distinctions motivate removing the
corresponding universal and causal interpretations from this paper.


\clearpage
\begin{thebibliography}{9}
\bibitem{kivi} Z. Liu et al.
KIVI: A Tuning-Free Asymmetric 2bit Quantization for KV Cache.
ICML, 2024. \url{https://arxiv.org/abs/2402.02750}.
\bibitem{palu} C.-C. Chang et al.
Palu: Compressing KV-Cache with Low-Rank Projection.
2024. \url{https://arxiv.org/abs/2407.21118}.
\bibitem{cskv} L. Wang et al.
CSKV: Training-Efficient Channel Shrinking for KV Cache in
Long-Context Scenarios. 2024.
\url{https://arxiv.org/abs/2409.10593}.
\bibitem{recalkv} X. Yan et al.
ReCalKV: Low-Rank KV Cache Compression via Head
Reordering and Offline Calibration. 2025.
\url{https://arxiv.org/abs/2505.24357}.
\bibitem{kqsvd} D. Lesens, B. T. Rakhshan, and G. Rabusseau.
KQ-SVD: Compressing the KV Cache with Provable Guarantees on
Attention Fidelity. 2025. \url{https://arxiv.org/abs/2512.05916}.
\end{thebibliography}
\end{document}